\documentclass[runningheads,a4paper]{llncs}

\usepackage{amssymb}
\usepackage{graphicx}
\usepackage{algorithm}
\usepackage{algpseudocode}

\usepackage{url}
\urldef{\mailsa}\path|{Shriprakash.Sinha@gmail.com}|
\newcommand{\keywords}[1]{\par\addvspace\baselineskip
\noindent\keywordname\enspace\ignorespaces#1}

\begin{document}

\mainmatter  

\title{Transductive-Inductive Cluster Approximation Via Multivariate Chebyshev Inequality}

\titlerunning{TI Cluster Approximation Via Multivariate Chebyshev Inequality}

%
%
\author{Shriprakash Sinha
}
\authorrunning{Sinha}

\institute{Tu Delft, Dept. of Mediamatics, Faculty of EEMCS, Mekelweg
  4,\\
2628 CD Delft, The Netherlands\\
\mailsa\\
}

%
%

\maketitle

\begin{abstract}
Approximating adequate number of clusters in multidimensional data is
an open area of research, given a level of compromise made on the
quality of acceptable results. The manuscript addresses the issue by
formulating a transductive inductive learning algorithm which uses
multivariate Chebyshev inequality. Considering clustering problem in
imaging, theoretical proofs for a particular level of compromise are
derived to show the convergence of the reconstruction error to a
finite value with increasing (a) number of unseen examples and (b) the
number of clusters, respectively. Upper bounds for these error rates
are also proved. Non-parametric estimates of these error from a random
sample of sequences empirically point to a stable number of
clusters. Lastly, the generalization of algorithm can be applied to
multidimensional data sets from different fields.
\keywords{Transductive Inductive Learning, Multivariate Chebyshev Inequality} 
\end{abstract}

\section{Introduction}
\label{sec:intro}
The estimation of clusters has been approached either via a batch
framework where the entire data set is presented and different
initializations of seed points or prototypes tested to find a model of
cluster that fits the data like in $k$-means \cite{Kanungo:2002} and
fuzzy $C$-means \cite{Bezdek:1981} or an online strategy clusters are approximated as
new examples of data are presented one at a time using variational
Dirichlet processes \cite{Gomes:2008} and incremental clustering based
on randomized algorithms \cite{Charikar:1997}. It is widely known that
approximation of adequate number of clusters using a multidimensional
data set is a open problem and a variety of solutions have been
proposed using Monte Carlo studies \cite{Dubes:1987},
Bayesian-Kullback learning scheme in mean squared error setting or
gaussian mixture \cite{Xu:1996}, model based approaches
\cite{Fraley:1998} and information theory \cite{Still:2004}, to cite a few. \par 
This work deviates from the general strategy of defining the number of
clusters apriori. It defines a level of compromise, tolerance or
confidence in the quality of clustering which gives an upper bound on
the number of clusters generated. Note that this is not at all similar
to defining the number of clusters. It only indicates the level of
confidence in the result and the requirement still is to estimate the
adequate number of clusters, which may be way below the bound. The
current work focuses on dealing with the issue of approximating the
number of clusters in an online paradigm when the confidence level has
been specified. In certain aspects it finds similarity with the
recent work on conformal learning theory \cite{Vovk:2005} and presents
a novel way of finding the approximation of cluster with a degree of
confidence.\par
Conformal learning theory \cite{Shafer:2008}, which
has its foundations in employing a transductive-inductive paradigm
deals with the idea of estimating the quality of predictions made on
the unlabeled example based on the already processed
data. Mathematically, given a set of already processed examples
($x_{1}, y_{1}$), ($x_{2}, y_{2}$), ..., ($x_{i-1}, y_{i-1}$), the
conformal predictors give a point prediction $\hat{y}$ for the unseen
example $x_{i}$ with a confidence level of
$\Gamma^{\varepsilon}$. Thus it estimates the confidence in the
quality of prediction using the original label $y_{i}$ after the
prediction has been made and before moving on to the next unlabeled
example. These predictions are made on the basis of a non-conformity
measure which checks how much the new example is different from a bag
of already seen examples. A bag is considered to be a
finite sequence $\mathcal{Z}$ ($z_{1}, z_{2}, ..., z_{i-1}$) of
examples, where $z_{i} = (x_{i}, y_{i})$. Then using the idea of
\emph{exchangeability}, it is known from \cite{Vovk:2005}, under a
weaker assumption that for every positive integer $i$, every
permutation $\pi$ of $\{1,2,...,i\}$, and every measurable
set $E \subset \mathcal{Z}^{i}$, the probability
distribution $P\{(z_{1}, z_{2}, ...) \in \mathcal{Z}^{\infty}$ :
$(z_{1}, z_{2}, ..., z_{i}) \in E\}$ = $P\{(z_{1}, z_{2}, ...) \in
\mathcal{Z}^{\infty}$ : $(z_{\pi(1)}, z_{\pi(2)}, ..., z_{\pi(i)}) \in
E\}$. A prediction for the new example $x_{i}$ is made if and only if
the frequency (p-value) of exchanging the new example with another
example in the bag is above certain value. \par
This manuscript finds its motivation from the foregoing theory of online
prediction using transductive-inductive paradigm. The research work
applies the concept of coupling the creation of new clusters via
transduction and aggregation of examples into these clusters via
induction. It finds its similarity with \cite{Shafer:2008} in
utilizing the idea of prediction region defined by a certain level of
confidence. It presents a simple algorithm that differs significantly
from conformal learning in the following aspect: (1) Instead of
working with sequences of data that contain labels, it works on
unlabeled sequences. (2) Due the first formulation, it becomes
imperative to estimate the number of clusters which is not known
apriori and the proposed algorithm comes to rescue by employing a
Chebyshev inequality. The inequality helps in providing an upper bound
on the number of clusters that could be generated on a random sample
of sequence. (3) The quality of the prediction in conformal learning
is checked based on the p-values generated online. The current
algorithms relaxes this restriction in checking the quality online and
just estimates the clusters as the data is presented. (4) The
foregoing step makes the algorithm a weak learner as it is sequence
dependent. To take stock of the problem, a global solution to the
adequate number of cluster is approximated by estimating kernel
density estimates on a sample of random sequences of a data. Finally,
the level of compromise captured by a parameter in the inequality
gives an upper bound on the number of clusters generated. In case of
clustering in static images, for a particular parameter value,
theoretical proofs show that the reconstruction error converges to a
finite value with increasing (a) number of unseen examples and (b) the
number of clusters. Empirical kernel density estimates of
reconstruction error over a random sample of sequences on toy examples
indicate the number of clusters that have high probability of low
reconstruction error. \emph{It is not necessary that labeled data are
always present to compute the reconstruction error}. In that case the
proposed algorithm stops short at density estimation of approximated
number of clusters from a random sequence of examples, with a certain
degree of confidence.\par
Another dimension of the proposed work is to use the generalization of
multivariate formulation of Chebyshev inequality \cite{Berge:1938},
\cite{Lal:1955}, \cite{Marshall:1960}. It is known that Chebyshev inequality
helps in proving the convergence of random sequences of different
data. Also the multivariate formulation of the Chebyshev inequality
facilitates in providing bounds for multidimensional data which is
often afflicted by the curse of dimensionality making it difficult to
compute multivariate probabilities. One of the generalizations that
exist for multivariate Chebyshev inequality is the consideration of
probability content of a multivariate normal random vector to lie in
an Euclidean $n$-dimensional ball \cite{Monhor:2005},
\cite{Monhor:2007}. This work employs a more conservative approach is
the employment of the Euclidian $n$-dimensional ellipsoid which
restricts the spread of the probability content \cite{Chen:2007}. Work
by \cite{Kanungo:2002} and \cite{Sinha:2011} provide motivation in
employment of multivariate Chebyshev inequality. \par
Efficient implementation and analysis of $k$-means clustering using the multivariate Chebyshev
inequality has been shown in \cite{Kanungo:2002}. The current work
differs from $k$-means in (1) providing an online setting to the
problem of clustering (2) estimating the number of  clusters for a
particular sequence representation of the same data via convergence
through ellipsoidal multivariate Chebyshev inequality, given the level of confidence,
compromise or tolerance in the quality of results (3) generating
global approximations of number of clusters from non-parametric
estimates of reconstruction error rates for sample of random sequences
representing the same data and (4) not fixing the cluster number
apriori. It must be noted that in $k$-means, the solutions may be
different for different initializations for a particular value of $k$
but the value of $k$ as such remains fixed. In the proposed work, with
a high probability, an estimate is made regarding the minimum number
of clusters that can represent the data with low reconstruction
error. This outlook broadens the perspective of finding multiple
solutions which are upper bounded as well as approximating a
particular number of cluster which have similar solutions. This
similarity in solutions for a particular number of cluster is
attributed to the constraint imposed by the Chebyshev inequality. A
key point to be noted is that using increasing levels of compromise or
confidence as in conformal learners, the proposed work generates a
nested set of solutions. Low confidence or compromise levels generate
tight solutions and vice versa. Thus the proposed weak learner
provides a set of solutions which are robust over a sample. \par
This manuscript also extends the work of \cite{Sinha:2011} on
employment of multivariate Chebyshev inequality for image representation. Work in
\cite{Sinha:2011} presents a hybrid model based on Hilbert space
filling curve \cite{Hilbert:1891} to traverse through the image. Since this
curve preserve the local information in the neighbourhood of a pixel,
it reduces the burden of good image representation in lower
dimensions. On the other side, it acts as a constraint on processing
the image in a particular fashion. The current work removes this
restriction of processing images via the space filling curves by
considering any pixel sequence that represents the image under
consideration. Again, a single sequence may not be adequate enough for the
learner to synthesize the image to a recongnizable level. This
can be attributed to the fact that in an unsupervised paradigm the
number of clusters are not known apriori and also the learner
would be sequence dependent. To reiterate, the proposed work addresses
the issues of $\bullet$ \emph{recognizability}, by defining a level of
compromise that a user is willing to make via the Chebyshev parameter
$\mathcal{C}_{p}$ and $\bullet$ \emph{sequence specific solution}, by
taking random samples of pixel sequences from the same image. The
latter helps in estimating a population dependent solution which would
be robust and stable synthesis. Regularization of these error over
approximated number of clusters for different levels of compromise
leads to an adequate number of clusters that synthesize the image with
minimal deviation from the original image. \par
Thus the current work provides a new perspective in approximation of
cluster number at a particular confidence level. To test the
propositions made, the problem of clustering in images is taken into
account. Generalizations of the algorithm can be made and applied to
different fields involving multidimensional data sets in an online
setting. Let $\mathcal{I}$ be an RGB image. A pixel in $\mathcal{I}$
is an example $x_{i}$ with $\mathcal{N}$ dimensions (here $\mathcal{N}$ =
3). It is assumed that examples appear randomly without repetition for
the proposed unsupervised learner. Note that when the sample space
(here the image $\mathcal{I}$) has finite number of examples in it
(here $\mathcal{M}$ pixels), then the total number of unique sequences
is $\mathcal{M}!$. When $\mathcal{M}$ is large, $\mathcal{M}!
\rightarrow \infty$. Currently, the algorithm works on a subset of unique
sequences sampled from $\mathcal{M}!$ sequences. The probability of a
sequence to occur is equally likely (in this case
$1/\mathcal{M}$). RGB images from the Berkeley Segmentation Benchmark
(BSB) \cite{Martin:2001} have been taken into consideration for the
current study. \par
\section{Transductive-Inductive Learning Algorithm}
\label{sec:tila}
Given that the examples ($z_{i} = x_{i}$) in a sequence appear
randomly, the challenge is to (1) learn the association of a
particular example to existing clusters or (2) create a new cluster,
based on the information provided by already processed examples. The
current algorithm handles the two issues via (1) evaluation of a
nonconformity measure defined by multivariate Chebyshev's inequality
formulation and (2) $1$-Nearest Neighbour (NN) transductive learning,
respectively. The multivariate formulation of the generalized Chebyshev
inequality \cite{Chen:2007} is applied to a new example using a single
Chebyshev parameter. This inequality tests the deviation of the new
example from the mean of a cluster of examples and gives a lower
probabilistic bound on whether the example belongs to the cluster
under investigation. If the new random example passes the test, then
it is associated with the cluster and the mean and covariance matrix
for the cluster is recomputed. In case there exists more than one
cluster which qualify for association, then the cluster with lowest
deviation to the new example is picked up for association. It is also
possible to assign the new example to a random chosen cluster from the
selected clusters to induce noise and then check for the
approximations on the number of cluster. This has not been considered
in the current work for the time being. In case
of failure to find any association, the algorithm employs $1$-NN
transductive algorithm to find a closest neighbour of the current
example under processing. This neighbour together with the current
example forms a new cluster. \par
\begin{algorithm}[!t]
\caption{Unsupervised Learner} \label{alg:final}
\begin{algorithmic}[1]
\Procedure{Unsupervised Learner}{$img$, $\mathcal{C}_{p}$}
\State [$nrows$, $ncols$] $\gets$ size($img$)
\State $\mathcal{M}$ $\gets$ $nrows \times ncols$ \Comment{Total no. of unseen examples}
\State $pt_{cntr}$ $\gets$ 0 \Comment{Number of examples encountered}
\State $pt_{idx}$ $\gets$ $\{1, 2, ..., \mathcal{M}\}$ \Comment{Total no. of indicies of unseen examples}
\Statex Initialize Variables
\State $cluster_{cntr}$ $\gets$ 0 \Comment{Number of clusters}
\State $CumErr_{val}$ $\gets$ 0 \Comment{Cummulative value}
\State $Err_{1} \gets$ [] \Comment{Error rate as no. of examples increase}
\State $Err_{2} \gets$ [] \Comment{Error rate as no. of clusters increase}
\While{Card($pt_{idx}$) examples remain unprocessed} \Comment{$pt_{idx} \subset \{1, 2, ..., \mathcal{M}\}$}
     \State Choose a random example $x_{i}$ s.t. $i \in pt_{idx}$
     \State $pt_{cntr}$ $\gets$ $pt_{cntr} + 1$
     \State Update $pt_{idx}$ i.e $pt_{idx}$ $\gets$ $pt_{idx} - \{i\}$
     \State $CRITERION$ $\gets$ []
     \State $Err_{val}$ $\gets$ 0
     \State $\forall q$ clusters were $q \in \{0, 1, 2, ..., cluster_{cntr}\}$
     \State \hspace{0.5cm} $Err_{val}$ $\gets$ $\sum_{k = 1}^{\ell} (x_{k} - \mathbf{E}_{q}(x))^{2}$  \Comment{$x$ means all examples in cluster $q$}
     \State \hspace{0.5cm} $CumErr_{val}$ $\gets$ $CumErr_{val} + Err_{val}$
     \State \hspace{0.5cm} Compute $\mathcal{D} \gets (x_{i} - \mathbf{E}_{q}(x))^{T}\Sigma_{q}^{-1}(x_{i} - \mathbf{E}_{q}(x))$
     \State \hspace{0.5cm} If $\mathcal{D}_{q} < \mathcal{C}_{p}$ \Comment{$\mathcal{C}_{p}$ is Chebyshev parameter}
     \State \hspace{1cm} $CRITERION \gets [CRITERION; \mathcal{D}_{q},
     q]$
    \State If more than one cluster that associates to $x_{i}$, i.e $length(CRITERION) \geq 1$
     \State \hspace{0.5cm} Associate  $x_{i}$ to selected cluster $q$ with minimum $\mathcal{D}_{q}$
    \State $Err_{1}$ $\gets$ [$Err_{1}, Err_{val}/pt_{cntr}$]
     \State If $x_{i}$ is not associated with any cluster, i.e $sum(FOUND) == 0$
     \State \hspace{0.5cm} $Err_{2}$ $\gets$ [$Err_{2}, Err_{val}/pt_{cntr}$]
     \State \hspace{0.5cm} $cluster_{cntr}$ $\gets$ $cluster_{cntr} + 1$
     \State \hspace{0.5cm} Using $1$-NN find $x_{j}$ closest to $x_{i}$ s.t. $j \in pt_{idx}$
     \State \hspace{0.5cm} Update $pt_{idx}$ i.e $pt_{idx}$ $\gets$ $pt_{idx} - \{j\}$
     \State \hspace{0.5cm} Form a new cluster $\{x_{i}, x_{j}\}$
\EndWhile
\EndProcedure
\end{algorithmic}
\end{algorithm}
Several important implications arise due to the usage of a
probabilistic inequality measure as a nonconformal measure. These will
be elucidated in detail in the later sections. An important point to
consider here is the usage of $1$-NN algorithm to create a new
cluster. Even though it is known that $1$-NN suffers from the problem
of the curse of dimensionality, for problems with small dimensions, it
can be employed for transductive learning. The aim of the proposed
work is not to address the curse of dimensionality issue. Also, note
that in the general supervised conformal learning algorithm, a
prediction has to be made before the next random example is
processed. This is not the case in the current unsupervised framework
of the conformal learning algorithm. In case the current random
example fails to associate with any of the existing clusters, under
the constraint yielded by the Chebyshev parameter, the NN helps in
finding the closest example (in feature space) from the remaining
unprocessed sample data set, to form a new cluster. Thus the formation
of a new cluster depends on the strictness of the Chebyshev parameter
$\mathcal{C}_{p}$. The procedure for unsupervised conformal learning
is presented in algorithm \ref{alg:final}. It does not strictly follow
the idea of finding confidence on the prediction as labels are not
present to be tested against. The goal here is to reconstruct the
clusters from a single pixel sequence such that they represent the
image. The quality of the reconstruction is taken up later on when a
random sample of pixel sequences are used to estimate the probability
density of the reconstruction error rates. Note that in the algorithm,
$\mathbf{E}_{q}(x)$ represents the mean of the examples $x$ in the
$q^{th}$ cluster and $\Sigma_{p}$ is the covariance matrix of $N$D
feature examples of the $q^{th}$ cluster. \par 
\section{Theoretical Perspective}
\label{sec:tp}
Application of the multivariate Chebyshev
inequality that yields a probabilistic bound enforces certain
important implications with regard to the clusters that are
generated. For the purpose of elucidation of the algorithm, the
starfish image is taken from \cite{Martin:2001}. 
\begin{figure}[!t]
\begin{center}
\includegraphics[width=11cm,height=14cm]{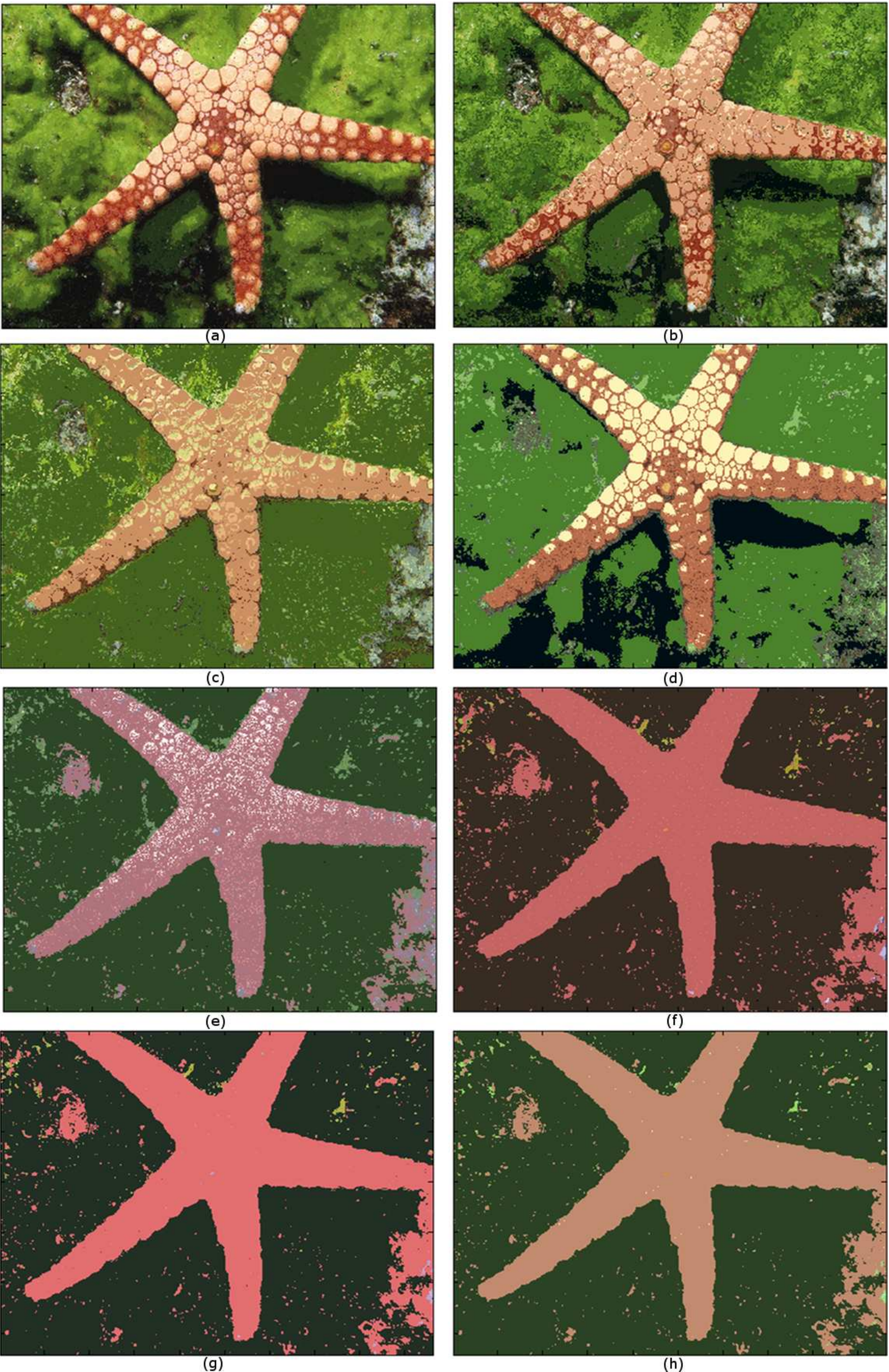}
\caption{A random sequence of Starfish Image segmented via unsupervised conformal learning algorithm. ($\mathcal{C}_{p}$, NoClust, TRErr) represent the tuple containing the Chebyshev paramenter ($\mathcal{C}_{p}$), number of clusters generated (NoClust) while using $\mathcal{C}_{p}$ and the total reconstruction error of the generated image from the original image (TRErr)(a) ($3, 1034, 17.746$), (b) ($5, 271, 36.32$), (c) ($7, 159, 54.71$), (d) ($9, 45, 40.591$), (e) ($11, 31, 62.606$), (f) ($13, 29, 66.061$), (g) ($15, 33, 65.424$), (h) ($17, 24, 64.98$).}
\label{fig:clust-starfish}
\end{center}
\end{figure}
\subsection{Multivariate Chebyshev Inequality}\label{sec:mci}
Let $X$ be a stochastic variable in $\mathcal{N}$ dimensions with a mean $E[X]$. Further, $\Sigma$ be the covariance matrix of all observations, each containing $\mathcal{N}$ features and $\mathcal{C}_{p} \in \mathcal{R}$, then the multivariate Chebyshev Inequality in \cite{Chen:2007} states that: 
\begin{eqnarray}
\mathcal{P} \{(X - E[X])^{T} \Sigma^{-1} (X - E[X]) \geq \mathcal{C}_{p}\} & \leq &
\frac{\mathcal{N}}{\mathcal{C}_{p}} \nonumber\\
\mathcal{P}\{(X - E[X])^{T} \Sigma^{-1} (X - E[X]) < \mathcal{C}_{p}\} & \geq &
1 - \frac{\mathcal{N}}{\mathcal{C}_{p}} \nonumber\\
\label{equ:ti}
\end{eqnarray}
i.e. the probability of the spread of the value of $X$ around the
sample mean $E[X]$ being greater than $\mathcal{C}_{p}$, is less than
$\mathcal{N}/\mathcal{C}_{p}$. There is a minor variation for the
univariate case stating that the probability of the spread of the
value of $x$ around the mean $\mu$ being greater than
$\mathcal{C}_{p}\sigma$ is less than $1/\mathcal{C}_{p}^{2}$. Apart
from the minor difference, both formulations convey the same message
about the probabilistic bound imposed when a random vector or number
$X$ lies outside the mean of the sample by a value of
$\mathcal{C}_{p}$. \par
\subsection{Association to Clusters}
Once a cluster is initialized (say with $x_{i}$ and $x_{j}$), the size
of the cluster depends on the number of examples getting associated
with it. The multivariate formalism of the Chebyshev inequality
controls the degree of uniformity of feature values of examples that
constitute the cluster. The association of the example to a cluster
happens as follows: Let the new random example (say $x_{t}$) be
considered for checking the association to a cluster. If the spread of
example $x_{t}$ from $\mathbf{E}_{q}(x)$ (the mean of the
\textbf{$q^{th}$} cluster $\{x_{i}, x_{j}\}$), factored by the
covariance matrix $\Sigma_{q}$, is below $\mathcal{C}_{p}$, then
$x_{t}$ is considered as a part of the cluster. Using Chebyshev
inequality, it boils down to: 
\begin{eqnarray}
\mathcal{P} \{(x_{t} - \mathbf{E}_{q}[x_{i}, x_{j}])^{T} \Sigma_{q}^{-1} (x_{t} -
\mathbf{E}_{q}[x_{i}, x_{j}]) \geq \mathcal{C}_{p}\} & \leq &
\frac{\mathcal{N}}{\mathcal{C}_{p}} \nonumber\\
\mathcal{P} \{(x_{t} - \mathbf{E}_{q}[x_{i}, x_{j}])^{T} \Sigma_{q}^{-1} (x_{t} -
\mathbf{E}_{q}[x_{i}, x_{j}]) < \mathcal{C}_{p}\} & \geq & 1 -
\frac{\mathcal{N}}{\mathcal{C}_{p}} \nonumber\\
\label{equ:ti_imp}
\end{eqnarray}
Satisfaction of this criterion suggests a possible cluster to which
$x_{t}$ could be associated. This test is conducted for all the
existing clusters. If there are more than one cluster to which $x_{t}$
can be associated, then the cluster which shows the minimum deviation
from the new random point is chosen. Once the cluster is chosen, its size is extended to
by one more example i.e. $x_{t}$. The cluster now constitutes
$\{x_{i}, x_{j}, x_{t}\}$. If no association is found at all, a new
cluster is initialized and the process repeats until all unseen
examples have been processed. The satisfaction of the inequality gives
a lower probabilistic bound on size of cluster by a value of $1 -
(\mathcal{N}/\mathcal{C}_{p})$, if the second version of the Chebyshev
formula is under consideration. Thus the size of the clusters grow
under a probabilistic constraint in a homogeneous manner. For a highly
inhomogeneous image, a cluster size may be very restricted or small
due to big deviation of pixel intensities from the cluster it is being
tested with. \par
Once the pixels have been assigned to respective decompositions, all
pixels in a single decomposition are assigned the average value of
intensities of pixels that constitute the decomposition. Thus is done
under the assumption that decomposed clusters will be homogeneous in
nature with the degree of homogeneity controlled by
$\mathcal{C}_{p}$. Figure \ref{fig:clust-starfish} shows the results of
clustering for varying values of $\mathcal{C}_{p}$ for the starfish
image from \cite{Martin:2001}. \par 
\subsection{Implications}\label{sec:implications}
In \cite{Sinha:2010} various implications have been proposed for using
multivariate Chebyshev inequality for image representation using space
filling curve. In order to extend on the their work, a few implications are
reiterated for further development. The inequality being a criterion,
the probability associated with the same gives a belief based bound on
the satisfaction of the criterion. In order to proceed, first a
definition of \emph{Decomposition} is needed. 
\begin{definition}
Let $\mathcal{D}$ be a \textbf{decomposition} which contains a set of
points $x$ with a mean of $\mathbf{E}_{q}(x)$. The set expands by
testing a new point $x_{t}$ via the Chebyshev inequality $\mathcal{P}
\{(x_{t} - \mathbf{E}_{q}(x))^{T}$ $\Sigma_{q}^{-1}$ $(x_{t} -
\mathbf{E}_{q}(x)) < \mathcal{C}_{p}\}$ $\geq 1 -
\frac{\mathcal{N}}{\mathcal{C}_{p}}$.
\end{definition}
The decomposition may include the point $x_{t}$ depending on the
outcome of the criterion. A point to be noted is that, if the new
point $x_{t}$ belongs to $\mathcal{D}$, then D can be represented as
$(x_{t} - \mathbf{E}_{q}(x))^{T}$ $\Sigma_{q}^{-1}$ $(x_{t} -
\mathbf{E}_{q}(x))$. 
\begin{lemma}
Decompositions $\mathcal{D}$ are bounded by lower probability bound of
$1 - (\mathcal{N}/\mathcal{C}_{p})$. \label{lem:01}
\end{lemma}
\begin{lemma}
The value of $\mathcal{C}_{p}$ reduces the size of the sample from
$\mathcal{M}$ to an upper bound of $\mathcal{M}/\mathcal{C}_{p}$
probabilistically with a lower bound of $1 -
(\mathcal{N}/\mathcal{C}_{p})$. Here $\mathcal{M}$ is the number of
examples in the image. \label{lem:02}
\end{lemma}
\begin{lemma}
As $\mathcal{C}_{p} \rightarrow \mathcal{N}$ the lower probability
bound drops to zero, implying large number of small decompositions
$\mathcal{D}$ can be achieved. Vice versa for $\mathcal{C}_{p}
\rightarrow \infty$. \label{lem:03}
\end{lemma}
It was stated that the image can be reconstructed from pixel sequences
at a certain level of compromise. From lemma \ref{lem:02}, it can be
seen that $\mathcal{C}_{p}$ reduces the sample size while inducing a
certain amount of error due to loss of information via averaging. This
reduction in sample size indicates the level of compromise at which
the image is to processed. This reduction in sample size or level of
compromise is directly related to the construction of probabilistically bounded
decompositions also. Since the decompositions are generated via the
usage of $\mathcal{C}_{p}$ in equation \ref{equ:ti}, the belief of
their existence in terms of a lower probability bound (from lemma
\ref{lem:01}) suggests a confidence in the amount of error incurred in
reconstruction of the image. For a particular pixel, this
reconstruction error can be computed by squaring the difference
between the value of the intensity in the original image and the
intensity value assigned after clustering. Since a somewhat
homogeneous decomposition is bounded probabilistically, the
reconstruction error of pixels that constitute it are also bounded
probabilistically. Thus for all decompositions, the summation of reconstruction
errors for all pixels is bounded. The bound indicates the confidence in
the generated reconstruction error. Also, by lemma \ref{lem:02}, since
the number of decompositions or clusters is upper bounded, the total
reconstruction error is also upper bounded. It now remains to be
proven that for a particular level of compromise, the error rates
converge as the number of processed examples and the number of
clusters increase. \par
In algorithm \ref{alg:final}, three error rates are computed as the
random sequence of examples get processed. For each original pixel $x_{i} \in
\mathcal{R}^{\mathcal{N}}$ in the image, let $x_{i}^{R}$ be the intensity
value assigned after clustering. Then the reconstruction error for
pixel $x_{i}$ is norm-2 $||x_{i} - x_{i}^{R}||_{2}$. Since a pixel
is assigned to a particular decomposition $\mathcal{D}_{q}$, it gets a
value of the mean of the all pixels that constitute the decomposition
$\mathcal{D}_{q}$. Thus the reconstruction error for a pixel turns out
to be $||x_{i} - \mathbf{E}_{q}(x)||_{2}$. For each cluster $q$, the
reconstruction error is $Err_{\mathcal{D}_{q}} = \sum_{i = 1}^{n} || x_{i} -
\mathbf{E}_{q}(x)||_{2}$. Note that the error also indicates how much
the examples deviate from the mean of their respective cluster. As new
examples are processed based on the information present from the
previous examples, the total error computed at after processing the
first $pt_{cntr}$ examples in a random sequence is $Err_{val} =
\sum_{q = 1}^{cluster_{cntr}} Err_{\mathcal{D}_{q}}$. The error rate
for these $pt_{cntr}$ examples is $Err_{1} =
Err_{val}/pt_{cntr}$. Finally, an error rate is computed that captures
how the deviation of the examples from their respective cluster means
happen, after the formation of a new cluster. This error is denoted by
$Err_{2}$. The formula for $Err_{2}$ is the same as $Err_{1}$ but
with a minute change in conception. The $Err_{val}$ are divided by the
total number of point processed after the formation of every new cluster. \par
\begin{figure}[!t]
\begin{center}
\includegraphics[width=8.5cm,height=6.5cm]{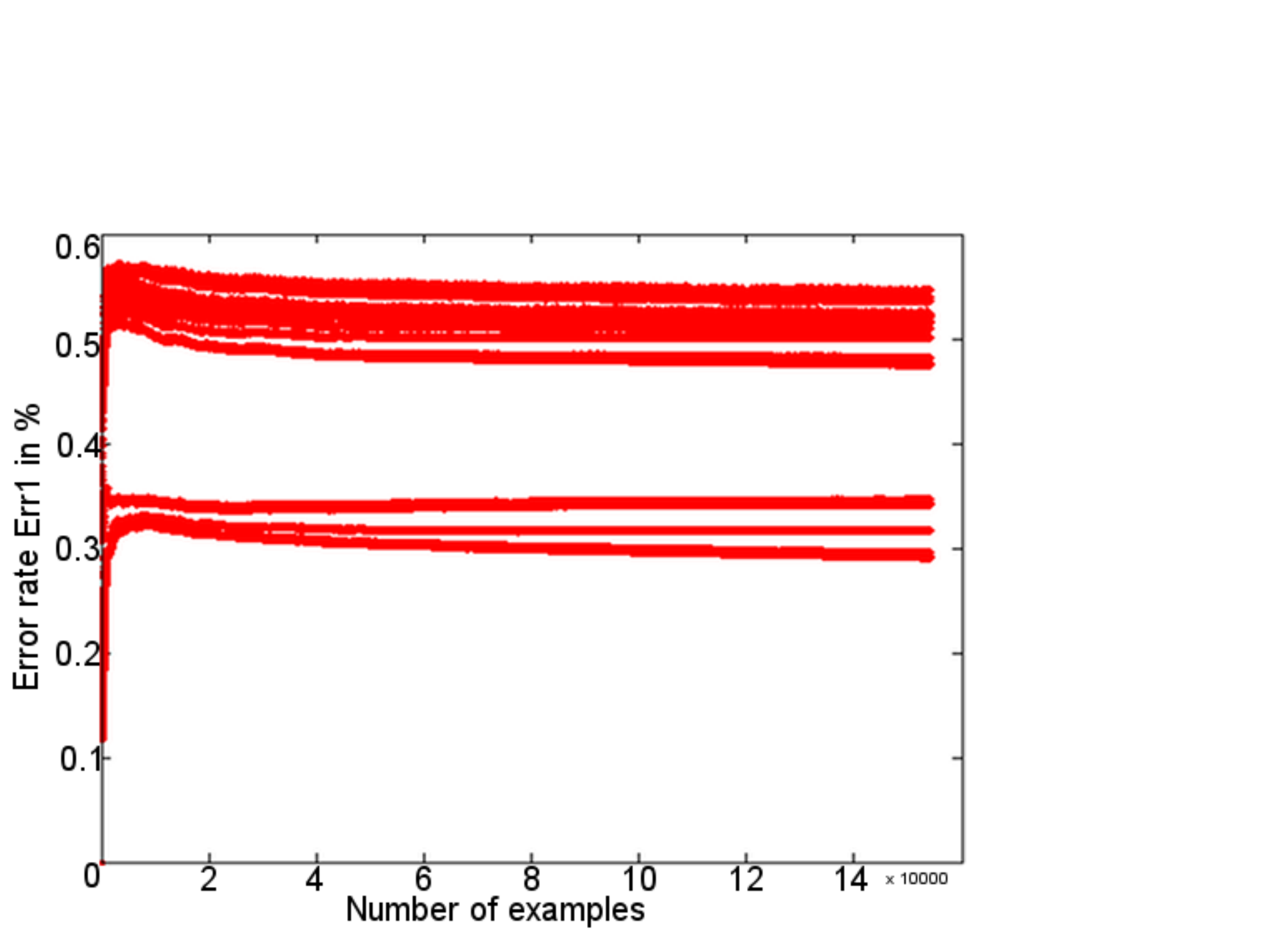}
\caption{Error rate $Err_{1}$ for a particular sequence with
  increasing number of examples with $\mathcal{C}_{p} = 7$.}
\label{fig:Err_1}
\end{center}
\end{figure}
\begin{theorem}
Let $\mathcal{Z}_{i}$ be a random sequence that represents the entire
image $\mathcal{I}$. If $\mathcal{Z}_{i}$ is decomposed into clusters
via the Chebyshev Inequality using the unsupervised learner, then the
reconstruction error rate $Err_{1}$ converges asymptotically with a probabilistically
lower bound or confidence level of $1 - \mathcal{N}/\mathcal{C}_{p}$
or greater. \label{thm:03}
\end{theorem}
\begin{proof}
It is known that the total reconstruction error after $pt_{cntr}$
examples have been processed, is $Err_{val} =
\sum_{q = 1}^{cluster_{cntr}} Err_{\mathcal{D}_{q}}$. And the error
rate is $Err_{1} = Err_{val}/pt_{cntr}$. It is also known from
equation \ref{equ:ti} that an example is associated to a particular
decomposition $\mathcal{D}_{q}$ if it satisfies the constraint $(x_{t}
- \mathbf{E}_{q}(x))^{T}$ $\Sigma_{q}^{-1}$ $(x_{t} -
\mathbf{E}_{q}(x)) < \mathcal{C}_{p}$. Since $\mathcal{C}_{p}$ defines
level of compromise on the image via lemma \ref{lem:02} and the
decompositions $\mathcal{D}_{q}$ is almost homogeneous, all examples
that constitute a decomposition have similar attribute values. Due to this
similarity between the attribute values, the non-diagonal elements of
the covariance matrix in the inequality above approach to zero. Thus,
$\Sigma_{q}^{-1} \approx det|\Sigma_{q}^{-1}|\mathbf{I}$, were $\mathbf{I}$ is the identity
matrix. The inequality then equates to: 
\begin{eqnarray}
(x_{t} - \mathbf{E}_{q}(x))^{T} det|\Sigma_{q}^{-1}|\mathbf{I} (x_{t} -
\mathbf{E}_{q}(x)) & \lessapprox & \mathcal{C}_{p} \nonumber \\
det|\Sigma_{q}^{-1}| (x_{t} - \mathbf{E}_{q}(x))^{T} \mathbf{I} (x_{t} -
\mathbf{E}_{q}(x)) & \lessapprox & \mathcal{C}_{p} \nonumber \\
(x_{t} - \mathbf{E}_{q}(x))^{T} \mathbf{I} (x_{t} -
\mathbf{E}_{q}(x)) & \lessapprox & \frac{\mathcal{C}_{p}}{det|\Sigma_{q}^{-1}|} \nonumber \\
||x_{t} - \mathbf{E}_{q}(x))||_{2} & \lessapprox &
\frac{\mathcal{C}_{p}}{det|\Sigma_{q}^{-1}|} \label{equ:exrecbnd} 
\end{eqnarray}
Thus, if $x_{i} = x_{t}$ was the last example to be associated to a
decomposition, the reconstruction error $||x_{i} -
\mathbf{E}_{q}(x)||$ for that example would be
upper bounded be
$\frac{\mathcal{C}_{p}}{det|\Sigma_{q}^{-1}|}$. Consequently, the
total error after processing $pt_{cntr}$ examples is also upper
bounded, i.e.
\begin{eqnarray}
Err_{val} &  &  = \sum_{q = 1}^{cluster_{cntr}} Err_{\mathcal{D}_{q}}
\nonumber \\
&  &  = \sum_{q = 1}^{cluster_{cntr}} \sum_{i = 1}^{n}||x_{i} -
\mathbf{E}_{q}(x)||_{2} \nonumber \\
&  & \lessapprox \sum_{q = 1}^{cluster_{cntr}} \sum_{i =
  1}^{n}\frac{\mathcal{C}_{p}}{det|\Sigma_{q}^{-1}|} \nonumber \\
&  & \lessapprox \sum_{q = 1}^{cluster_{cntr}} \sum_{i =
  1}^{n} \frac{\mathcal{C}_{p}}{det|\Sigma_{q}^{-1}|} \label{equ:errrecbnd}
\end{eqnarray}
Thus the error rate $Err_{1} = Err_{val}/pt_{cntr}$ is also upper
bounded. Different decompositions may have different
$\Sigma_{q}^{-1}$, but in the worst case scenario, if the
decomposition with the lowest covariance is substituted for every
other decompositions, then the upper bound on the error is
$\frac{\mathcal{C}_{p}}{det|\Sigma_{lowest}^{-1}| \times pt_{cntr}}
\Sigma_{q=1}^{clust_{cntr}} \Sigma_{i=1}^{n} (1)$ which equates to
$\frac{\mathcal{C}_{p}}{det|\Sigma_{lowest}^{-1}|}$. \qed
\end{proof}
It is important to note that this error rate converges to a finite
value asymptotically as the number of processed examples
increases. This is because initially when the learner has not seen
enough examples to learn and solidify the knowledge in terms of a
stable mean and variance of decompositions, the error rate $Err_{1}$
increases as new examples are presented. This is attributed to the
fact that new clusters are formed more often in the intial stages, due
to lack of prior knowledge. After a certain time, when large number of
examples have been encountered to help solidify the knowledge or
stabilize the decompositions, then addition of further examples does
not increment the error. This stability of clusters is checked via the
multivariate formulation of the Chebyshev inequality in equation
\ref{equ:ti_imp}. The stability also casues the error rate $Err_{1}$
to stabilize and thus indicate its convergence in a bounded manner
with a probabilistic confidence level. Thus for any value of
$pt_{cntr}$, there exists an upper bound on reconstruction error,
which stabilizes as $pt_{cntr}$ increases. \par
For $\mathcal{C}_{p} = 7$, the image (c) in figure
\ref{fig:clust-starfish} shows the clustered image that is generated
using the unsupervised conformal learning algorithm. Pixels in a
cluster of the generated image have the mean of the cluster as their
intensity value or the label. This holds for all the clusters in the
generated image. The total number of clusters generated for a
particular random sequence was $159$. The error rate $Err_{1}$ is
depicted in figure \ref{fig:Err_1}. \par 
\begin{theorem}
Let $\mathcal{Z}_{i}$ be a random sequence that represents the entire
image $\mathcal{I}$. If $\mathcal{Z}_{i}$ is decomposed into clusters
via the Chebyshev Inequality using the unsupervised learner, then the
reconstruction error rate $Err_{2}$ converges asymptotically with a
probabilistically lower bound or confidence level of $1 -
\mathcal{N}/\mathcal{C}_{p}$ or greater. \label{thm:04} 
\end{theorem}
\begin{figure}[!t]
\begin{center}
\includegraphics[width=8.5cm,height=6.5cm]{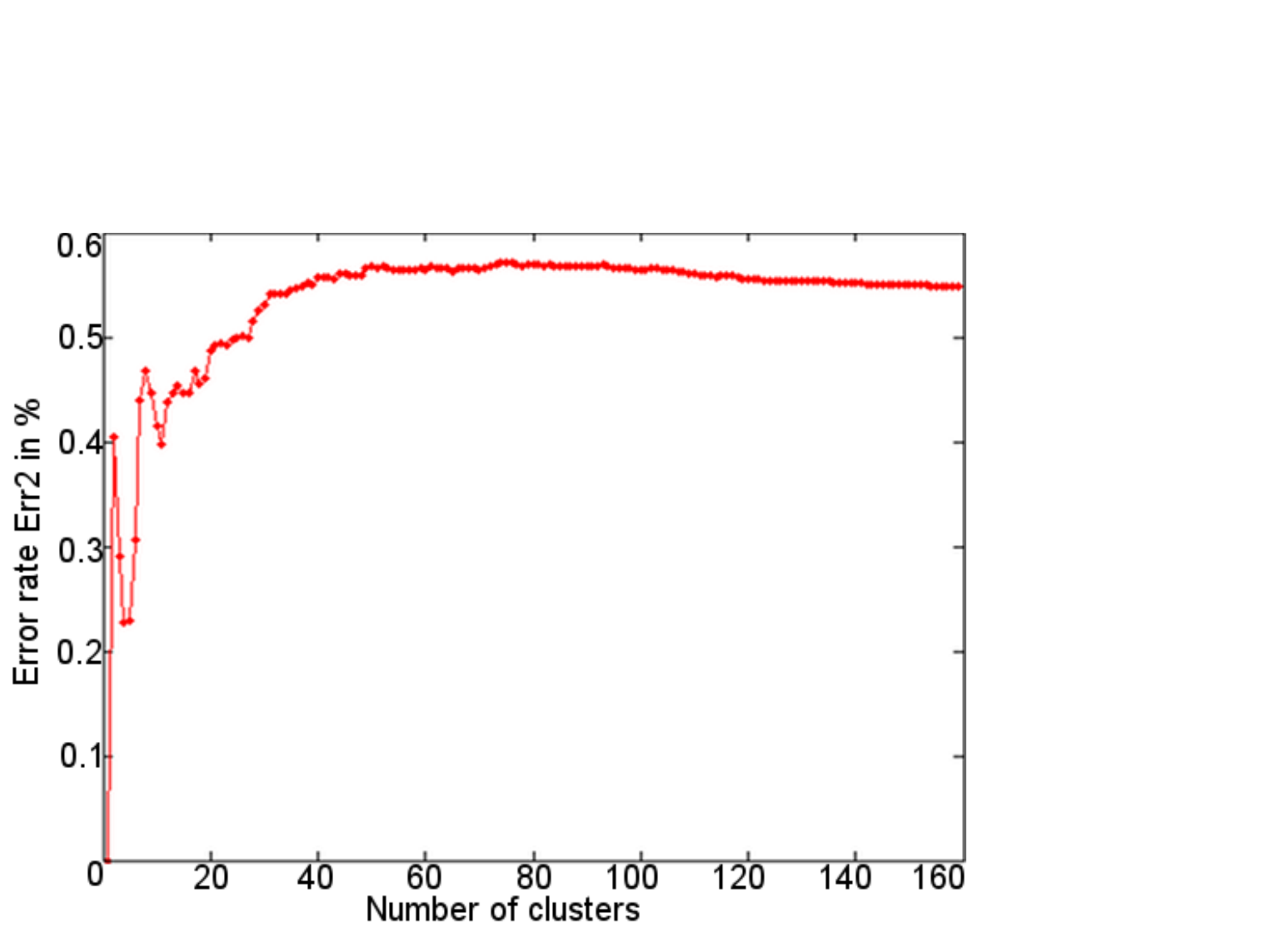}
\caption{Error rate $Err_{2}$ for a particular sequence with increasing number of clusters and $\mathcal{C}_{p} = 7$.}
\label{fig:Err_2}
\end{center}
\end{figure}
\begin{proof}
The error rate $Err_{2}$ is the computation of error after each new
cluster is formed. The upper bound on $Err_{2}$ as
the number of clusters or decompositions increase follows a proof
similar to one presented in theorem \ref{thm:03}.. \qed
\end{proof}
Again for the same $\mathcal{C}_{p} = 7$, the image (c) in figure
\ref{fig:clust-starfish}, the error rate $Err_{2}$ is depicted in
figure \ref{fig:Err_2}. Intuitively, it can be seen that both the
reconstruction error rates converge to an approximately similar
value. \par
The theoretical proofs and the lemmas suggest that, for a given level
of compromise $\mathcal{C}_{p}$ there exists an upper bound on the
reconstruction error as well as the number of clusters. But this
reconstruction error and the number of clusters is dependent on a
pixel sequence presented to the learner. Does this mean that for a
particular level of compromise one may find values of reconstruction
error and number of clusters that may never converge to a finite
value, when a random sample of pixel sequences that represent an
image are processed by the learner? Or in a more simplified way, is it
possible to find a reconstruction error and the number of clusters at
a particular level of compromise that best represents the image? This
points to the problem of whether an image can be reconstructed at a
particular level of compromise where there is a high probability of
finding a low reconstruction error and the number of clusters, from a
sample of sequences. \par
The existence of such a probability value would require the knowledge
of the probability distribution of the reconstruction error over
increasing (1) number of examples and (2) number of clusters
generated. In this work, kernel density estimation (KDE) is used to estimate
the probability distribution of the reconstruction error $Err_{1}$ and
$Err_{2}$. To investigate into the quality of the solution obtained,
the error rates were generated for different random sequences and a
KDE was evaluated on the observations. The density
estimate empirically point to the least error rates with high
probability. It was found that the error rates $Err_{1}$, $Err_{2}$
and the number of clusters, all converge to a particular value, for a
given image. \par 
\begin{figure}[!t]
\begin{center}
\includegraphics[width=7.5cm,height=12cm]{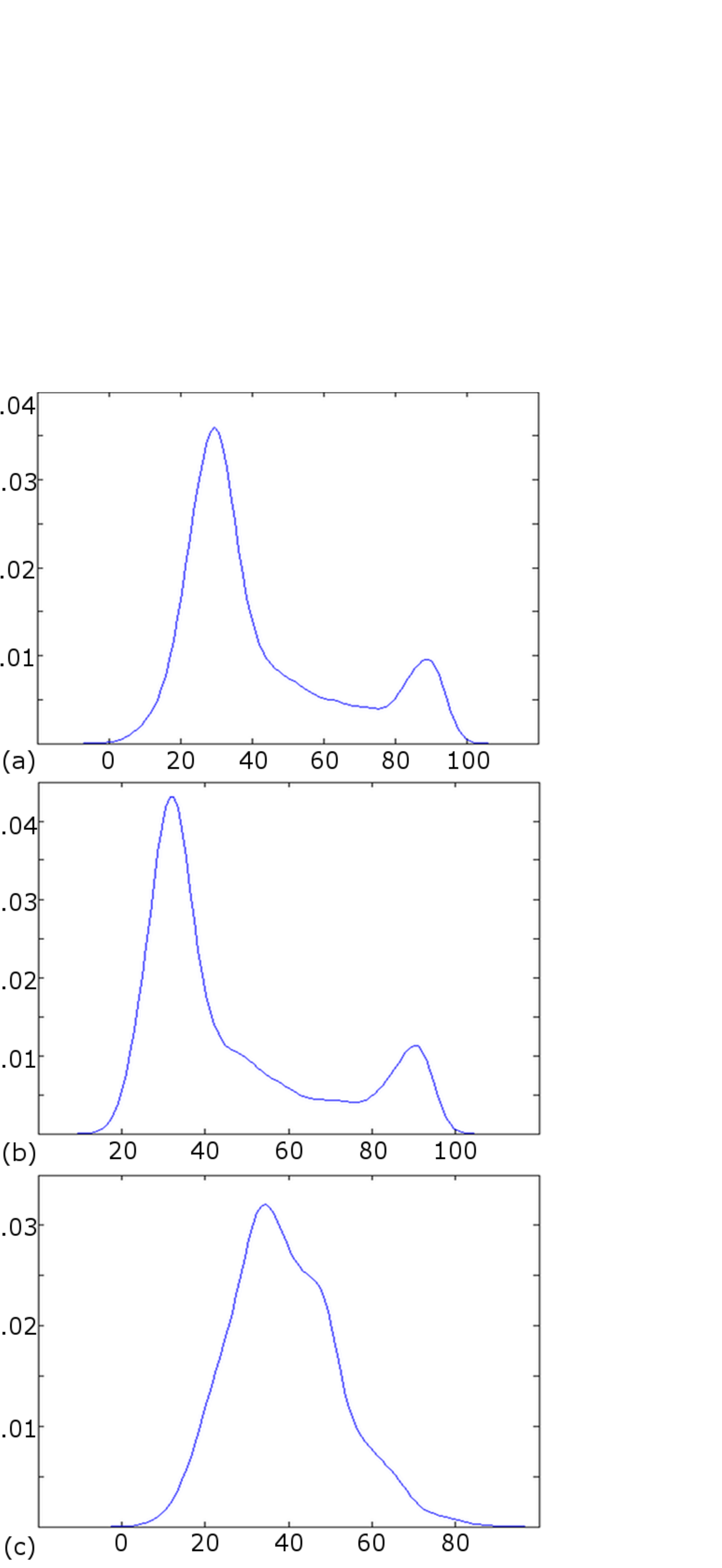}
\caption{The probability density estimates for (a) $Err_{1}$ (b) $Err_{2}$ and (c) the number of clusters obtained via the unsupervised conformal learner, generated over $1000$ random sequences representing the same image with $\mathcal{C}_{p} = 10$.}
\label{fig:kde-Cp-10}
\end{center}
\end{figure}
For $\mathcal{C}_{p} = 10$, the probability density estimates were
generated using the density estimates on error rates and the number of clusters
obtained on $1000$ random sequences of the same image. It was found
that the error rates $Err_{1}$, $Err_{2}$ and the number of clusters
converge to $33.1762$, $35.9339$ and $38$, respectively. Figure
\ref{fig:kde-Cp-10} shows the graphs for the same. It can be seen from graphs (a)
and (b) in figure \ref{fig:kde-Cp-10}, that both $Err_{1}$ and
$Err_{2}$ converge nearly to the similar values. \par 
It can been noted that with increasing value of the parameter
$\mathcal{C}_{p}$, the bound on the decomposition expands which
further leads to generation of lower number of clusters required to
reconstruct the image. Thus it can be expected that at lower levels of
compromise, the reconstruction error (via KDE) is low but the number
of clusters (via KDE) is very high and vice versa. Figure
\ref{fig:ReconstErr_vs_NoD} shows the behaviour of these
reconstruction error and number of clusters generated as the level of
compromise increases. High reconstruction error does not necessarily
mean that the representation of the image is bad. It only suggests the
granularity of reconstruction obtained. Thus the reconstruction of the
image can yield finer details at low level of compromise and point to
segmentations at high level of compromise. Regularization over the
level of compromise and the number of clusters would lead to a
reconstruction which has low reconstruction error as well as adequate
number of decompositions that represent an image properly. \par
\begin{figure}[!t]
\begin{center}
\includegraphics[width=8.5cm,height=5cm]{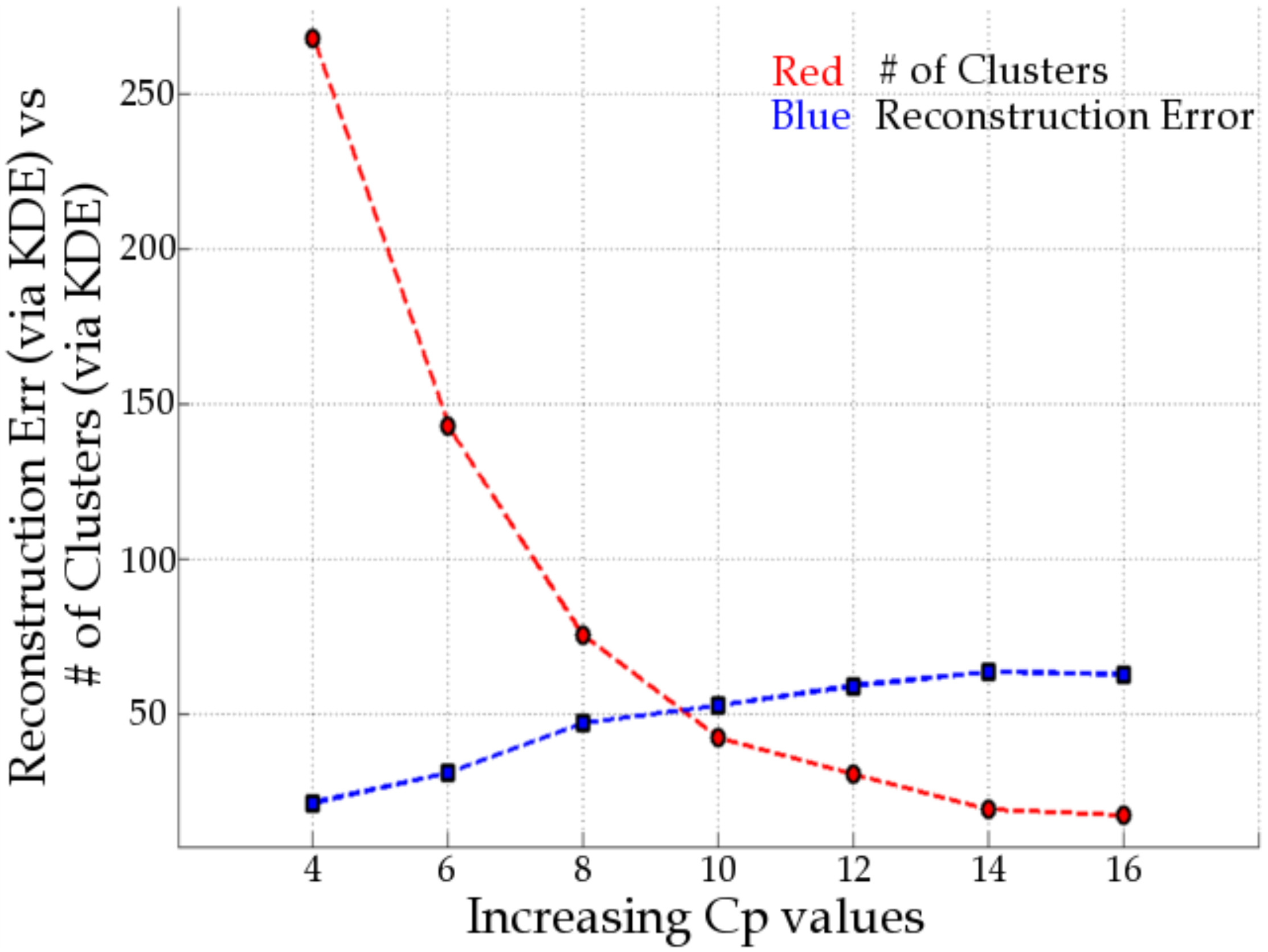}
\caption{Behaviour of reconstruction error (via KDE) and number of
  cluster or decompositions (via KDE) based on increasing values of
  $\mathcal{C}_{p}$.}
\label{fig:ReconstErr_vs_NoD}
\end{center}
\end{figure}
There are a few points that need to be remembered when applying such
an online learning paradigm. The reconstructed results come near to
original image only at a level of imposed compromise. As the size of
dataset or the image increases, the time consumed and the number of
computations involved for processing also increases. To start with,
the learner would perform well in clean images than on noisy
images. Adaptations need to be made for processing noisy images or the
pre-processing would be a necessary step before application of such an
algorithm. Other inequalities can also be taken into account for
multivariate information online. It would be tough to compare the
algorithm with other powerful clustering algorithms as the proposed
work presents a weak learner and provides a general solution with no
tight bounds on the quality of clustering. \par
Nevertheless, the current work contributes to estimation of cluster
number in an unsupervised paradigm using transductive-inductive
learning strategy. It can be said that for a fixed Chebyshev
parameter, in a bootstrapped sequence sampling environment without
replacement, the unsupervised learner converges to a finite error rate
along with the a finite number of clusters. The result in terms of
clustering and the error rates may not be the most optimal (where the
meaning depends on the goal of optimization), but it does give an
affirmative clue that image decomposition is robust and convergent. \par
\section{Conclusion}\label{sec:conclusion}
A simple transductive-inductive learning strategy for unsupervised
learning paradigm is presented with the usage of multivariate
Chebyshev inequality. Theoretical proofs of convergence in number of
clusters for a particular level of compromise show (1) stability of
result over a sequence and (2) robustness of probabilistically
estimated approximation of cluster number over a random sample of
sequences, representing the same multidimensional data. Lastly, upper bounds
generated on the number of clusters point to a limited search
space.
\par
%
%
\par

\bibliographystyle{plain}
\bibliography{arxiv_paperbib}

\begin{thebibliography}{10}

\bibitem{Berge:1938}
PO~Berge.
\newblock A note on a form of tohebycheff's theorem for two variables.
\newblock {\em Biometrika}, 29(3-4):405, 1938.

\bibitem{Bezdek:1981}
J.C. Bezdek.
\newblock {\em Pattern recognition with fuzzy objective function algorithms}.
\newblock Kluwer Academic Publishers, 1981.

\bibitem{Charikar:1997}
M.~Charikar, C.~Chekuri, T.~Feder, and R.~Motwani.
\newblock Incremental clustering and dynamic information retrieval.
\newblock In {\em Proceedings of the twenty-ninth annual ACM symposium on
  Theory of computing}, pages 626--635. ACM, 1997.

\bibitem{Chen:2007}
X.~Chen.
\newblock A new generalization of chebyshev inequality for random vectors.
\newblock {\em arXiv:math.ST}, 0707(0805v1):1--5, 2007.

\bibitem{Dubes:1987}
R.C. Dubes.
\newblock How many clusters are best?-an experiment.
\newblock {\em Pattern Recognition}, 20(6):645--663, 1987.

\bibitem{Fraley:1998}
C.~Fraley and A.E. Raftery.
\newblock How many clusters? which clustering method? answers via model-based
  cluster analysis.
\newblock {\em The computer journal}, 41(8):578, 1998.

\bibitem{Gomes:2008}
R.~Gomes, M.~Welling, and P.~Perona.
\newblock Incremental learning of nonparametric bayesian mixture models.
\newblock In {\em Computer Vision and Pattern Recognition, 2008. CVPR 2008.
  IEEE Conference on}, pages 1--8. Ieee, 2008.

\bibitem{Hilbert:1891}
D.~Hilbert.
\newblock Uber die stetige abbildung einer linie auf ein flachenstuck.
\newblock {\em Math. Ann.}, 38:459–460, 1891.

\bibitem{Kanungo:2002}
T.~Kanungo, D.M. Mount, N.S. Netanyahu, C.D. Piatko, R.~Silverman, and A.Y. Wu.
\newblock An efficient-means clustering algorithm: Analysis and implementation.
\newblock {\em IEEE Transactions on Pattern Analysis and Machine Intelligence},
  pages 881--892, 2002.

\bibitem{Lal:1955}
DN~Lal.
\newblock A note on a form of tchebycheff's inequality for two or more
  variables.
\newblock {\em Sankhy{\=a}: The Indian Journal of Statistics (1933-1960)},
  15(3):317--320, 1955.

\bibitem{Marshall:1960}
A.W. Marshall and I.~Olkin.
\newblock Multivariate chebyshev inequalities.
\newblock {\em The Annals of Mathematical Statistics}, pages 1001--1014, 1960.

\bibitem{Martin:2001}
D.~Martin, C.~Fowlkes, D.~Tal, and J.~Malik.
\newblock A database of human segmented natural images and its application to
  evaluating segmentation algorithms and measuring ecological statistics.
\newblock {\em Proc. 8th Int'l Conf. Computer Vision}, 2:416--423, July 2001.

\bibitem{Monhor:2007}
D.~Monhor.
\newblock A chebyshev inequality for multivariate normal distribution.
\newblock {\em Probability in the Engineering and Informational Sciences},
  21(02):289--300, 2007.

\bibitem{Monhor:2005}
D.~Monhor and S.~Takemoto.
\newblock Understanding the concept of outlier and its relevance to the
  assessment of data quality: Probabilistic background theory.
\newblock {\em Earth, Planets, and Space}, 57(11):1009--1018, 2005.

\bibitem{Shafer:2008}
G.~Shafer and V.~Vovk.
\newblock A tutorial on conformal prediction.
\newblock {\em The Journal of Machine Learning Research}, 9:371--421, 2008.

\bibitem{Sinha:2011}
S.~Sinha and G.~Horst.
\newblock Bounded multivariate surfaces on monovariate internal functions.
\newblock {\em IEEE International Conference on Image Processing},
  (18):1037--1040, 2011.

\bibitem{Still:2004}
S.~Still and W.~Bialek.
\newblock How many clusters? an information-theoretic perspective.
\newblock {\em Neural computation}, 16(12):2483--2506, 2004.

\bibitem{Vovk:2005}
V.~Vovk, A.~Gammerman, and G.~Shafer.
\newblock {\em Algorithmic learning in a random world}.
\newblock Springer Verlag, 2005.

\bibitem{Xu:1996}
L.~Xu.
\newblock How many clusters?: A ying-yang machine based theory for a classical
  open problem in pattern recognition.
\newblock In {\em Neural Networks, 1996., IEEE International Conference on},
  volume~3, pages 1546--1551. IEEE, 1996.

\end{thebibliography}

\end{document}